\documentclass[11pt,a4paper]{article}
\usepackage[hyperref]{acl2020}
\usepackage{times}
\usepackage{latexsym}
\usepackage{amsmath}
\usepackage{amssymb} 
\usepackage{enumerate}
\usepackage{booktabs}
\usepackage{multirow}
\usepackage{pgf}
\usepackage{subcaption}
\usepackage{graphicx}
\usepackage{amsthm}
\newtheorem{thm}{Proposition}

\newcommand*{\affaddr}[1]{#1} 
\newcommand*{\affmark}[1][*]{\textsuperscript{#1}}
\newcommand*{\email}[1]{\texttt{#1}}
\usepackage{microtype}

\aclfinalcopy 
\def\aclpaperid{1047} 

\title{Generative Semantic Hashing Enhanced via Boltzmann Machines}

\author{%
Lin Zheng\affmark[1], \quad Qinliang Su\affmark[1,2]\Thanks{ Corresponding author.} , \quad Dinghan Shen\affmark[3], \quad Changyou Chen\affmark[4]\\
\affaddr{\affmark[1]School of Data and Computer Science, Sun Yat-sen University}\\
\affaddr{\affmark[2]Guangdong Key Laboratory of Big Data Analysis and Processing, Guangzhou, China}\\
\affaddr{\affmark[3]Microsoft Dynamics 365 AI}\\
\affaddr{\affmark[4]CSE Department, SUNY at Buffalo}\\
\email{zhenglin6@mail2.sysu.edu.cn, suqliang@mail.sysu.edu.cn}\\
\email{dishen@microsoft.com, changyou@buffalo.edu}\\
}

\date{}
\begin{document}
\maketitle
\begin{abstract}
Generative semantic hashing is a promising technique for large-scale information retrieval thanks to its fast retrieval speed and small memory footprint. For the tractability of training, existing generative-hashing methods mostly assume a factorized  form for the posterior distribution, enforcing independence among the bits of hash codes. From the perspectives of both model representation and code space size, independence is always not the best assumption. In this paper, to introduce correlations among the bits of hash codes, we propose to employ the distribution of Boltzmann machine as the variational posterior. To address the intractability issue of training, we first develop an approximate method to reparameterize the distribution of a Boltzmann machine by augmenting it as a hierarchical concatenation of a Gaussian-like distribution and a Bernoulli distribution. Based on that, an asymptotically-exact lower bound is further derived for the evidence lower bound (ELBO). With these novel techniques, the entire model can be optimized efficiently. Extensive experimental results demonstrate that by effectively modeling correlations among different bits within a hash code, our model can achieve significant performance gains.
\end{abstract}

\section{Introduction}
\label{sec:intro}
Similarity search, also known as nearest-neighbor search, aims to find items that are similar to a query from a large dataset. It plays an important role in modern information retrieval systems and has been used in various applications, ranging from plagiarism analysis \citep{Plagiarized} to content-based multimedia retrieval \citep{contentbasedretrieval}, \emph{etc}. However, looking for nearest neighbors in the Euclidean space is often computationally prohibitive for large-scale datasets (calculating cosine similarity with high-dimensional vectors is computationally-expensive). Semantic hashing circumvents this problem by representing semantically similar documents with compact and {\emph{binary}} codes. Accordingly, similar documents can be retrieved by evaluating the hamming distances of their hash codes much more efficiently.

To obtain similarity-preserving hash codes, extensive efforts have been made to learn hash functions that can preserve the similarity information of original documents in the binary embedding space \citep{sh3,sh1}. Existing methods often require the availability of label information, which is often expensive to obtain in practice. To avoid the use of labels, generative semantic hashing methods have been developed. Specifically, the variational autoencoder (VAE) is first employed for semantic hashing in \citep{vdsh}, and their model is termed VDSH. As a two-step process, the continuous document representations obtained from VAE are directly converted into binary hash codes. To resolve the two-step training problem, Bernoulli priors are leveraged as the prior distribution in NASH \citep{nash}, replacing the continuous Gaussian prior in VDSH. By utilizing straight-through (ST) technique \citep{bengioST}, their model can be trained in an end-to-end manner, while keeping the merits of VDSH. Recently, to further improve the quality of hash codes, mixture priors are investigated in BMSH \citep{bmsh}, while more accurate gradient estimators are studied in Doc2hash \citep{doc2hash}, both under a similar framework as NASH. 

Due to the training-tractability issue, the aforementioned generative hashing methods all assume a factorized variational form for the posterior, {\it e.g.}, independent Gaussian in VDSH and  independent Bernoulli in NASH, BMSH and Doc2hash. This assumption prevents the models from capturing dependencies among the bits of hash codes. Although uncorrelated bits are sometimes preferred in hashing, as reported in \citep{decorr1}, this may not apply to generative semantic hashing. This is due to the fact that the independent assumption could severely limit a model's ability to yield meaningful representations and thereby produce high-quality hash codes. Moreover, as the code length increases (to {\it e.g.} 128 bits), the number of possible codes (or simply the code space) will be too large for a dataset with limited number of data points. As a result, we advocate that correlations among bits of a hash code should be considered properly to restrict the embedding space, and thus enable a model to work effectively under a broad range of code lengths.

To introduce correlations among bits of hash codes, we propose to adopt the Boltzmann-machine (BM) distribution \citep{bm1985} as a variational posterior to capture various complex correlations. One issue with this setting, relative to existing efficient training methods, is the inefficiency brought in training. To address this issue, we first prove that the BM distribution can be augmented as a hierarchical concatenation of a Gaussian-like distribution and a Bernoulli distribution. Using this result, we then show that samples from BM distributions can be well reparameterized easily. To enable efficient learning, an asymptotically-exact lower bound of the standard evidence lower bound (ELBO) is further developed to deal with the notorious problem of the normalization term in Boltzmann machines. With the proposed reparameterization and the new lower bound, our model can be trained efficiently as the previous generative hashing models that preserve no bit correlations. Extensive experiments are conducted to evaluate the performance of the proposed model. It is observed that on all three public datasets considered, the proposed model achieves the best performance among all comparable models. In particular, thanks to the introduced correlations, we observe the performance of the proposed model does not deteriorate as the code length increases. This is surprising and somewhat contrary to what has been observed in other generative hashing models.

\section{Preliminaries}
\paragraph{Generative Semantic Hashing}
In the context of generative semantic hashing, each document is represented by a sequence of words $x=\{w_1, w_2, \cdots, w_{|x|}\}$, where $w_i$ is the $i$-th word and is denoted by a $|V|$-dimensional one-hot vector; $|x|$ and $|V|$ denotes the document size (number of words) and the vocabulary size, respectively. Each document $x$ is modeled by a joint probability: 
\begin{equation}
	p_\theta(x, s)=p_\theta(x|s)p(s),
\end{equation}
where $s$ is a latent variable representing the document's hash code. With the probability $p_\theta(x, s)$ trained on a set of documents, the hash code for a document $x$ can be derived directly from the posterior distribution $p_\theta(s|x)$. In existing works, the likelihood function, or the decoder takes a form $p_\theta(x|s) =\prod_{i=1}^{|x|}p_\theta(w_i|s)$ with
\begin{equation} \label{decoder}
	p_\theta(w_i|s) \triangleq \frac{\exp(s^T E w_i + b_i)}{\sum_{j=1}^{|V|}\exp(s^T E e_j + b_j)},
\end{equation}
where $E \in {\mathbb{R}}^{m\times |V|}$ is the matrix connecting the latent code $s$ and the one-hot representation of words; and $e_j$ is the one-hot vector with the only `1' locating at the $i$-th position. Documents could be modelled better by using more expressive likelihood functions, {\it e.g.}, deep neural networks, but as explained in \cite{nash}, they are more likely to destroy the crucial distance-keeping property for semantic hashing. Thus, the simple form of \eqref{decoder} is often preferred in generative hashing. As for the prior distribution $p(s)$,  it is often chosen as the standard Gaussian distribution as in VDSH \cite{vdsh}, or the Bernoulli distribution as in NASH and BMSH \cite{nash, bmsh}.

\paragraph{Inference}
Probabilistic models can be trained by maximizing the log-likelihood $\log p_\theta(x)$ with $p_\theta(x) = \int_s{p_\theta(x, s)ds}$. However, due to the intractability of calculating $p_\theta(x)$, we instead optimize its evidence lower bound (ELBO), {\it i.e.}, 
\begin{equation} \label{ELBO}
	{\mathcal{L}} = \mathbb{E}_{q_{\phi}(s | x)}\left[\log \frac{p_{\theta}(x | s)p(s)}{q_{\phi}(s | x)}\right],
\end{equation}
where $q_{\phi}(s|x)$ is the proposed variational posterior parameterized by $\phi$. It can be shown that $\log p_\theta(x) \ge {\mathcal{L}}$ holds for any $q_{\phi}(s | x)$ , and that if $q_{\phi}(s | x)$ is closer to the true posterior $p_\theta(s|x)$, the bound ${\mathcal{L}}$ will be tighter. Training then reduces to maximizing the lower bound ${\mathcal{L}}$ w.r.t. $\theta$ and $\phi$. In VDSH \cite{vdsh}, $q_{\phi}(s | x)$ takes the form of an independent Gaussian distribution
\begin{equation} \label{posterior_Gauss}
	q_{\phi}(s | x)={\mathcal{N}}\left(s|\mu_\phi(x), \mathrm{diag}(\sigma_\phi^2(x))\right),
\end{equation}
where $\mu_\phi(x)$ and $\sigma_\phi(x)$ are two vector-valued functions parameterized by multi-layer perceptrons (MLP) with parameters $\phi$. Later, in NASH and BMSH \cite{nash, bmsh}, $q_{\phi}(s | x)$ is defined as an independent Bernoulli distribution, {\it i.e.},
\begin{equation} \label{posterior_Bernoulli}
	q_{\phi}(s | x)= {\mathrm{Bernoulli}}(g_\phi(x)),
\end{equation}
where $g_\phi(x)$ is also vector-valued function parameterized by a MLP. The value at each dimension represents the probability of being $1$ at that position. The MLP used to parameterize the posterior $q_{\phi}(s|x)$ is also referred to as the encoder network.

One key requirement for efficient end-to-end training of generative hashing method is the availability of reparameterization for the variational distribution $q_{\phi}(s | x)$. For example, when $q_{\phi}(s | x)$ is a Gaussian distribution as in \eqref{posterior_Gauss}, a sample $s$ from it can be efficiently reparameterized as
\begin{equation} \label{reparam_vdsh}
	s = \mu_\phi(x) + \sigma_\phi(x)\cdot \epsilon
\end{equation}
with $\epsilon \sim {\mathcal{N}}(0, I)$. When $q_{\phi}(s | x)$ is a Bernoulli distribution as in \eqref{posterior_Bernoulli}, a sample from it can be reparameterized as
\begin{equation} \label{reparam_nash}
	s = \frac{\mathrm{sign}\left(g_\phi(x) - \epsilon \right) +1 }{2}
\end{equation}
where  $\epsilon \in {\mathbb{R}}^m$ with elements $\epsilon_i \sim {\mathrm{uniform}}(0, 1)$. 
With these reparameterization tricks, the lower bound in \eqref{ELBO} can be estimated by the sample $s$ as 
\begin{equation} \label{ELBO_sample}
	{\mathcal{L}} \approx \log \frac{p_{\theta}(x | s_\phi)p(s_\phi)}{q_{\phi}(s_\phi | x)},
\end{equation}
where $s$ has been denoted as $s_\phi$ to explicitly indicate its dependence on $\phi$. To train these hashing models, the backpropagation algorithm can be employed to estimate the gradient of \eqref{ELBO_sample} w.r.t. $\theta$ and $\phi$ easily. However, it is worth noting that in order to use the reparameterization trick, all existing methods assumed a factorized form for the proposed posterior $q_{\phi}(s|x)$, as shown in \eqref{posterior_Gauss} and \eqref{posterior_Bernoulli}. This suggests that the binary bits in hash codes are independent of each other, which is not the best setting in generative semantic hashing.

\section{Correlation-Enhanced Generative Semantic Hashing}
\label{sec:method}
In this section, we present a scalable and efficient approach to introducing correlations into the bits of hash codes, by using a Boltzmann-machine distribution as the variational posterior with approximate reparameterization. 

\subsection{Boltzmann Machine as the Variational Posterior} \label{sec: reparam}
Many probability distributions defined over binary variables $s\in \{0,1\}^m$ are able to capture the dependencies. Among them, the most famous one should be the Boltzmann-machine distribution \citep{bm1985}, which takes the following form:
\begin{equation}\label{boltzmann_machine}
b(s) = \frac{1}{Z} e^{\frac{1}{2}s^T \Sigma s + \mu^Ts} ,
\end{equation}
where $\Sigma \in {\mathbb{R}}^{m\times m}$ and $\mu \in {\mathbb{R}}^m$ are the distribution parameters; and $Z \triangleq \sum_{s}e^{\frac{1}{2}s^T \Sigma s + \mu^Ts}$ is the normalization constant. The Boltzmann-machine distribution can be adopted to model correlations among the bits of a hash code. Specifically, by restricting the posterior to the Boltzmann form
\begin{equation} \label{boltzmann_posterior}
	q_\phi(s|x) = \frac{1}{Z_\phi} e^{-E_\phi(s)}
\end{equation}
and substituting it into the lower bound of \eqref{ELBO}, we can write the lower bound as:
\begin{equation} \label{ELBO_boltzmann}
	{\mathcal{L}} = {\mathbb{E}}_{q_\phi(s|x)} \!\! \left[\log \! \frac{p_{\theta}(x | s)p(s)}{e^{-E_\phi(s)}}\right]  + \log Z_\phi,
\end{equation}
where $E_\phi(s)\triangleq - \frac{1}{2}s^T \Sigma_\phi(x) s - \mu^T_\phi(x) s$; and $\Sigma_\phi(x)$ and $\mu_\phi(x)$ are functions parameterized by the encoder network with parameters $\phi$ and $x$ as input. One problem with such modeling is that the expectation term ${\mathbb{E}}_{q_\phi(s|x)}[\cdot]$ in \eqref{ELBO_boltzmann} cannot be expressed in a closed form due to the complexity of $q_\phi(s|x)$. Consequently, one cannot directly optimize the lower bound ${\mathcal{L}}$ w.r.t. $\theta$ and $\phi$. 

\subsection{Reparameterization}\label{ssec:3-2}
An alternative way is to approximate the expectation term by using the reparameterized form of a sample $s$ from $q_\phi(s|x)$, as was done in the previous {\emph{uncorrelated}} generative hashing models (see \eqref{reparam_vdsh} and \eqref{reparam_nash}). Compared to existing simple variational distributions, there is no existing work on how to reparameterize the complicated Boltzmann-machine distribution. To this end, we first show that the Boltzmann-machine distribution can be equivalently written as the composition of an approximate correlated Gaussian distribution and a Bernoulli distribution. 

\begin{thm}\label{prop1}
A Boltzmann-machine distribution $b(s) = \frac{1}{Z} e^{\frac{1}{2}s^T \Sigma s + \mu^T s}$ with $\Sigma \succ 0$ can be equivalently expressed as the composition of two distributions, that is,
\begin{equation} \label{marginal_posterior}
	b(s) = \int{p(s|r)p(r) dr},
\end{equation}
where $p(r) = \frac{1}{Z}\prod_{i=1}^{m}(e^{r_i}+1)\cdot {\mathcal{N}}(r;\mu, \Sigma)$; $p(s|r) = \prod_{i=1}^m p(s_i|r_i)$ with $s_i$ and $r_i$ denoting the $i$-th element of $s$ and $r$; and $p(s_i|r_i) \triangleq {\mathrm{Bernoulli}}(\sigma(r_i))$ with $\sigma(\cdot)$ being the sigmoid function.
\end{thm}
\begin{proof}
See Appendix \ref{appendix:proof_of_proposition_1} for details.
\end{proof}

Based on Proposition \ref{prop1}, we can see that a sample from the Boltzmann-machine distribution $q_\phi(s|x)$ in \eqref{boltzmann_posterior} can be sampled hierarchically as 
\begin{equation}
r\sim q_\phi(r|x) \quad \text{and} \quad s \sim {\mathrm{Bernoulli}}(\sigma(r)), 
\end{equation}
where 
\begin{equation} \label{Equ:q_phi}
q_{\phi}(r|x) \!\!=\!\! \frac{1}{Z}\prod_{i=1}^{m}(e^{r_i}+1)\cdot {\mathcal{N}}(r;\mu_\phi(x), \Sigma_\phi(x))
\end{equation}
and $\sigma(\cdot)$ is applied to its argument element-wise. From the expression of $q_\phi(r|x)$, we can see that for small values of $r_i$, the influence of $(e^{r_i}+1)$ on the overall distribution is negligible, and thus $q_\phi(r|x)$ can be well approximated by the Gaussian distribution ${\mathcal{N}}(r;\mu_\phi(x), \Sigma_\phi(x))$. For relatively large $r_i$, the term  $(e^{r_i}+1)$ will only influence the  distribution mean,  roughly shifting the Gaussian distribution ${\mathcal{N}}(r;\mu_\phi(x), \Sigma_\phi(x))$ by an amount approximately equal to its variance. For problems of interest in this paper, the variances of posterior distribution are often small, hence it is reasonable to approximate samples from $q_{\phi}(r|x)$ by those from ${\mathcal{N}}(r;\mu_\phi(x), \Sigma_\phi(x))$.

With this approximation, we can now draw samples from Boltzmann-machine distribution $q_\phi(s|x)$ in \eqref{boltzmann_posterior} approximately by the two steps below
\begin{align}
	r & \sim {\mathcal{N}}(r; \mu_\phi(x), \Sigma_\phi(x)), \\
	s & \sim {\mathrm{Bernoulli}}(\sigma(r)). \label{s_i_sample}
\end{align}
For the Gaussian sample $r  \sim {\mathcal{N}}(r; \mu_\phi(x), \Sigma_\phi(x))$, similar to \eqref{reparam_vdsh}, it can be reparameterized as
\begin{equation} \label{r_raparam}
	r = \mu_\phi(x) + L_\phi(x) \cdot \epsilon,
\end{equation}
where $L_\phi(x)$ is the Cholesky decomposition matrix of $\Sigma_\phi(x)$ with $\Sigma_\phi(x)=L_\phi(x)L^T_\phi(x)$; and $\epsilon\in {\mathbb{R}}^m$ with $\epsilon \sim {\mathcal{N}}(0, I)$. It should be noted that in practice, we can define the function $L_\phi(x)$ in advance and then obtain $\Sigma_\phi(x)$ as $\Sigma_\phi(x)=L_\phi(x)L^T_\phi(x)$, thus the Cholesky decomposition is not needed.

Given the Gaussian sample $r$, similar to the reparameterization of Bernoulli variables in \eqref{reparam_nash}, we can reparameterize the Bernoulli sample $s \sim {\mathrm{Bernoulli}}(\sigma(r))$ as
$
	s = \frac{{\mathrm{sign}(\sigma(r) - u) + 1}}{2},
$
where $u \in {\mathbb{R}}^m$ with each element $u_i \sim {\mathrm{uniform}}(0, 1)$. By combining the above reparameterizations, a sample from the Boltzmann-machine distribution $q_\phi(s|x)$ can then be approximately reparameterized as
\begin{equation} \label{s_reparam}
	s_\phi = \frac{{\mathrm{sign}\left(\sigma(\mu_\phi(x) \!+ \! L_\phi(x) \cdot \epsilon) \!- \! u\right) \!+ \! 1}}{2},
\end{equation}
where the subscript $\phi$ is to explicitly indicate that the sample $s$ is expressed in terms of $\phi$.

With the reparameterization $s_\phi$, the expectation term in \eqref{ELBO_boltzmann} can be approximated as $\log \! \frac{p_{\theta}(x | s_\phi)p(s_\phi)}{e^{-E_\phi(s_\phi)}}$. Consequently, the gradients of this term w.r.t. both $\theta$ and $\phi$ can be evaluated efficiently by backpropagation, with the only difficulty lying at the non-differentiable function $\mathrm{sign}(\cdot)$ of $s_\phi$ in \eqref{s_reparam}. Many works have been devoted to estimate the gradient involving discrete random variables \citep{bengioST,gumbel1,gumbel2,rebar,relax,arm}. Here, we adopt the simple straight-through (ST) technique \cite{bengioST}, which has been found performing well in many applications. By simply treating the hard threshold function $\mathrm{sign}(\cdot)$ as the identity function, the ST technique estimates the gradient as
\begin{align}\label{grad}
  \frac{\partial s_\phi}{\partial \phi}
  \approx \frac{1}{2}\frac{\partial \left[\sigma(\mu_{\phi}(x) + L_{\phi}(x) \epsilon) - u\right]}{\partial \phi}.
\end{align}
Then, the gradient of the first term in ELBO ${\mathcal{L}}$ w.r.t. $\phi$ can be computed efficiently by backpropagation.

\subsection{An Asymptotically-Exact Lower Bound}
\label{sec:new_bound}
To optimize the ELBO in \eqref{ELBO_boltzmann}, we still need to calculate the gradient of $\log Z_{\phi}$, which is known to be notoriously difficult. A common way is to estimate the gradient $\frac{\partial \log Z_\phi}{\partial \phi}$ by MCMC methods \cite{pcd,logz,su2017probabilistic,su2017unsupervised}, which are computationally expensive and often of high variance. By noticing a special form of the ELBO \eqref{ELBO_boltzmann}, we develop a lower bound for the ELBO ${\mathcal{L}}$, where the $\log Z_{\phi}$ term can be conveniently \textit{cancelled out}. Specifically, we introduce another probability distribution $h(s)$ and lower bound the original ELBO:
\begin{equation}
  \label{newlb}
  {\mathcal{\widetilde L}} =  {\mathcal{L}} - \mathrm{KL}(h(s)||q_{\phi}(s|x)).
\end{equation}
Since $\mathrm{KL}(\cdot)\ge 0$, we have $\widetilde{\mathcal{L}}(\theta, \phi) \le  {\mathcal{L}}$ holds for all $h(s)$, {\it i.e.}, ${\mathcal{\widetilde L}}$ is a  lower bound of ${\mathcal{L}}$, and equals to the ELBO $ {\mathcal{L}}$ when $h(s) = q_\phi(s|x)$. For the choice of $h(s)$, it should be able to reduce the gap between ${\mathcal{{\widetilde L}}}$ and ${\mathcal{L}}$ as much as possible, while ensuring that the optimization is tractable. Balancing on the two sides, a mixture distribution is used
\begin{equation}\label{sk}
  h_k(s)  = \frac{1}{k}\sum_{i=1}^k p(s|r^{(i)}),
\end{equation}
where $k$ denotes the number of components; $p(s|r^{(i)})$ is the multivariate Bernoulli distribution and $r^{(i)}$ is the $i$-th sample drawn from $q_\phi(r|x)$ as defined in \eqref{Equ:q_phi}.
By substituting $h_k(s)$ into \eqref{newlb} and taking the expectation w.r.t. $r^{(i)}$, we have 
\begin{align} \label{L_k}
	{\mathcal{\widetilde L}}_k \!\! \triangleq \! {\mathcal{L}} \!- \! \mathbb{E}_{q_\phi(r^{(1\cdots k)}|x) } \!\! \left[\mathrm{KL}\! \left(h_k(s) || q_{\phi}(s|x)\right)\right]
\end{align}
where $q_\phi(r^{(1\cdots, k)}|x ) = \prod_{i=1}^k q_\phi(r^{(i)}|x)$.
It can be proved that the bound ${\mathcal{\widetilde L}}_k$ gradually approaches the ELBO ${\mathcal{L}}$ as $k$ increases, and finally equals to it as $k\to \infty$. Specifically, we have
\begin{thm} \label{prop_lowerbound}
For any integer $k$, the lower bound ${\mathcal{\widetilde L}}_k$ of the ELBO satisfies the conditions: 1) ${\mathcal{\widetilde L}}_{k+1} \geq {\mathcal{\widetilde L}}_k$; 2) $\lim_{k \to \infty}{\mathcal{\widetilde L}}_k = {\mathcal{L}}$.
\end{thm}
\begin{proof}
	See Appendix \ref{appendix:proofofpos2}  for details.
\end{proof}

By substituting ${\mathcal{L}}$ in \eqref{ELBO_boltzmann} and $h_k(s)$ in \eqref{sk} into \eqref{L_k}, the bound can be further written as
\begin{align} \label{newlb_noZ}
	{\mathcal{{\widetilde L}}}_k & =  {\mathbb{E}}_{q_\phi(s|x)} \!\! \left[\log \! \frac{p_{\theta}(x | s)p(s)}{e^{-E_\phi(s)}}\right] \nonumber \\
	 & \;\;\; -  {\mathbb{E}}_{q_\phi(r^{(1\cdots k)}|x)} \!\! \left[ {\mathbb{E}}_{h_k(s)} \!\! \left[\! \log \! \frac{h_k(s)}{e^{-E_\phi(s)}} \!\right] \right],
\end{align}
where the $\log Z_\phi$ term is cancelled out since it appears in both terms but has opposite signs. For the first term in \eqref{newlb_noZ}, as discussed at the end of Section \ref{sec: reparam}, it can be approximated as $\log \frac{p_{\theta}(x | s_\phi)p(s_\phi)}{e^{-E_\phi(s_\phi)}}$. For the second term,  each sample $r^{(i)}$ for $i=1, \cdots, k$ can be approximately reparameterized like that in \eqref{r_raparam}. Given the $r^{(i)}$ for $i=1, \cdots, k$,  samples from $h_k(s)$ can also be reparameterized in a similar way as that for Bernoulli distributions in \eqref{reparam_nash}. Thus, samples drawn from $r^{(1\cdots k)} \sim q_\phi(r^{(1\cdots k)}|x)$ and $s \sim h_k(s)$ are also reparameterizable, as detailed in Appendix \ref{appendix:mixturereparam}. By denoting this reparametrized sample as $\tilde s_\phi$, we can approximate the second term in \eqref{newlb_noZ} as $\log \frac{h_k(\tilde s_\phi)}{e^{-E_\phi(\tilde s_\phi)}} $. Thus the lower bound \eqref{newlb_noZ} becomes
\begin{equation}
	{\mathcal{\widetilde L}}_k \approx \log \frac{p_{\theta}(x | s_\phi)p(s_\phi)}{e^{-E_\phi(s_\phi)}} - \log \frac{h_k(\tilde s_\phi)}{e^{-E_\phi(\tilde s_\phi)}}.
\end{equation}
With the discrete gradient estimation techniques like the ST method, the gradient of ${\mathcal{\widetilde L}}_k$ w.r.t. $\theta$ and $\phi$ can then be evaluated efficiently by backpropagation. Proposition \ref{prop_lowerbound} indicates that the exact ${\mathcal{\widetilde L}}_k$ gets closer to the ELBO as $k$ increases, so better bound can be expected for the approximated ${\mathcal{\widetilde L}}_k$ as well when $k$ increases. In practice, a moderate value of $k$ is found to be sufficient to deliver a good performance.

\subsection{Low-Rank Perturbation for the Covariance Matrix}
\label{ssec:lowrank}
In the reparameterization of a Gaussian sample, $r_\phi = \mu_\phi(x) + L_\phi(x) \cdot \epsilon$ in \eqref{r_raparam}, a $m\times m$ matrix $L_\phi(x)$ is required, with $m$ denoting the length of hash codes. The elements of $L_\phi(x)$ are often designed as the outputs of neural networks parameterized by $\phi$. Therefore, if $m$ is large, the number of neural network outputs will be too large. To overcome this issue, a more parameter-efficient strategy called {\it Low-Rank Perturbation} is employed, which restricts covariance matrix to the form
\begin{equation}\label{lr}
  \Sigma = D + UU^{\top},
\end{equation}
where $D$ is a diagonal matrix with positive entries and $U=[u_1, u_2, \cdots u_v]$ is a low-rank perturbation matrix with $u_{i} \in \mathbb{R}^{m}$ and $v\ll m$. Under this low-rank perturbed $\Sigma$, the Gaussian samples can be reparameterized as
\begin{equation}\label{reparam_lowrank}
    r_\phi = \mu_\phi(x) + D^{1/2}_\phi(x) \cdot \epsilon_1 + U_\phi(x) \cdot \epsilon_2,
\end{equation}
where $\epsilon_1 \sim \mathcal{N}\left(0, I_m\right)$ and $\epsilon_2 \sim \mathcal{N}\left(0, I_v\right)$. We can simply replace \eqref{r_raparam} with the above expression in any place that uses $r$. In this way, the number of neural network outputs can be dramatically reduced from $m^2$ to $mv$.

\section{Related Work}
\label{sec:related}
Semantic Hashing \citep{srbm} is a promising technique for fast approximate similarity search. Locality-Sensitive Hashing, one of the most popular hashing methods \citep{lsh}, projects documents into low-dimensional hash codes in a randomized manner. However, the method does not leverage any information of data, and thus generally performs much worse than those data-dependent methods. Among the data-dependent methods, one of the mainstream methods is supervised hashing, which learns a function that could output similar hash codes for semantically similar documents by making effective use of the label information \citep{sh3,sh1}.

Different from supervised methods, unsupervised hashing pays more attention to the intrinsic structure of data, without making use of the labels. Spectral hashing \citep{spectralHashing}, for instance, learns balanced and uncorrelated hash codes by seeking to preserve a global similarity structure of documents. Self-taught hashing \citep{sth}, on the other hand, focuses more on preserving local similarities among documents and presents a two-stage training procedure to obtain such hash codes. In contrast, to generate high-quality hash codes, iterative quantization \citep{itq} aims to minimize the quantization error, while maximizing the variance of each bit at the same time.

Among the unsupervised hashing methods, the idea of generative semantic hashing has gained much interest in recent years. Under the VAE framework, VDSH \citep{vdsh} was proposed to first learn \textit{continuous} the documents' latent representations, which are then cast into binary codes. While semantic hashing is achieved with generative models nicely, the \textit{two-stage} training procedure is problematic and is prone to result in local optima. To address this issue, NASH \citep{nash} went one step further and presented an integrated framework to enable the \textit{end-to-end} training by using the discrete Bernoulli prior and the ST technique, which is able to estimate the gradient of functions with discrete variables. Since then, various directions have been explored to improve the performance of NASH. \citep{bmsh} proposed to employ the mixture priors to improve the model's capability to distinguish documents from different categories, and thereby improving the quality of hash codes. On the other hand, a more accurate gradient estimator called Gumbel-Softmax \citep{gumbel1,gumbel2} is explored in Doc2hash \citep{doc2hash} to replace the ST estimator in NASH. More recently, to better model the similarities between different documents, \citep{rankingnash} investigated the combination of generative models and ranking schemes to generate hash codes. Different from the aforementioned generative semantic hashing methods, in this paper, we focus on how to incorporate correlations into the bits of hash codes.

\section{Experiments}
\label{sec:ex}
\subsection{Experimental Setup}
\label{ssec:exsetup}
\paragraph{Datasets}
Following previous works, we evaluate our model on three public benchmark datasets:
i) Reuters21578, which consists of 10788 documents with 90 categories; 
ii) 20Newsgroups, which contains 18828 newsgroup posts from 20 different topics;
iii) TMC, which is a collection of 21519 documents categorized into 22 classes.

\paragraph{Training Details}
For the conveniences of comparisons, we use the same network architecture as that in NASH and BMSH. Specifically, 
a 2-layer feed-forward neural network with 500 hidden units and a ReLU activation function
is used as an inference network, which receives the TF-IDF of a document as input and outputs the mean and covariance matrix of the Gaussian random variables $r$. During training, the dropout \citep{dropout} is used to alleviate the overfitting issue, with the keeping probability selected from \{0.8, 0.9\} based on the performance on the validation set. The Adam optimizer \citep{adam} is used to train our model, with the learning rate set to 0.001 initially and then decayed for every 10000 iterations. For all experiments on different datasets and lengths of hash codes, the rank $v$ of matrix $U$ is set to 10 and the number of component $k$ in the distribution $h_k(s)$ is set to 10 consistently, although a systematic ablation study is conducted in Section \ref{ssec:analysis} to investigate their impacts on the final performances.

\paragraph{Baselines}
The following unsupervised semantic hashing baselines are adopted for comparisons:
Locality Sensitive Hashing (LSH) \citep{lsh}, 
Stack Restricted Boltzmann Machines (S-RBM) \citep{srbm}, 
Spectral Hashing (SpH) \citep{spectralHashing}, 
Self-Taught Hashing (STH) \citep{sth},
Variational Deep Semantic Hashing (VDSH) \citep{vdsh},
Neural Architecture for Generative Semantic Hashing (NASH) \citep{nash}, and
Semantic Hashing model with a Bernoulli Mixture prior (BMSH)\citep{bmsh}.

\paragraph{Evaluation Metrics}
The performance of our proposed approach is measured by retrieval precision {\it i.e.}, the ratio of the number of relevant documents to that of retrieved documents. A retrieved document is said to be relevant if its label is the same as that of the query one. Specifically, during the evaluating phase, we first pick out top 100 most similar documents for each query document according to the hamming distances of their hash codes, from which the precision is calculated. The precisions averaged over all query documents are reported as the final performance.

\subsection{Results of Generative Semantic Hashing}
\label{ssec:results}

The retrieval precisions on datasets TMC, Reuters and 20Newsgroups are reported in Tables \ref{tb:tmc}, \ref{tb:Reuters} and \ref{tb:20Newsgroups}, respectively, under different lengths of hash codes. Compared to the generative hashing method NASH without considering correlations, we can see that the proposed method, which introduces correlations among bits by simply employing the distribution of Boltzmann machine as the posterior, performs significantly better on all the three datasets considered. This strongly corroborates the benefits of taking correlations into account when learning the hash codes. From the tables, we can also observe that the proposed model even outperforms the BMSH, an enhanced variant of NASH that employs more complicated mixture distributions as a prior. Since only the simplest prior is used in the proposed model, larger performance gains can be expected if mixture priors are used as in BMSH. Notably, a recent work named RBSH is proposed in \citep{rankingnash}, which improves NASH by specifically ranking the documents according to their similarities. However, since it employs a different data preprocessing technique as the existing works, we cannot include its results for a direct comparison here. Nevertheless, we trained our model on their preprocessed datasets and find that our method still outperforms it. For details about the results, please refer to Appendix \ref{appendix:rbsh}.

Moreover, when examining the retrieval performance of hash codes under different lengths, it is observed that the performance of our proposed method never deteriorates as the code length increases, while other models start to perform poorly after the length of codes reaching a certain level. For the most comparable methods like VDSH, NASH and BMSH, it can be seen that the performance of 128 bits is generally much worse than that of 64 bits. This phenomenon is illustrated more clearly in Figure \ref{fig:1}. This may attribute to the reason that for hash codes without correlations, the number of codes will increase exponentially as the code length increases. Because the code space is too large, the probability of assigning similar items to nearby binary codes may decrease significantly. But for the proposed model, since the bits of hash codes are correlated to each other, the effective number of codes can be determined by the strength of correlations among bits, effectively restricting the size of code space. Therefore, even though the code length increases continually, the performance of our proposed model does not deteriorate.

\begin{table}
	\centering
	\resizebox{\columnwidth}{!}{    
		\begin{tabular}{c||c|c|c|c|c}
			
			\toprule
			{Method}   & {8 bits} & {16 bits} & {32 bits} & {64 bits} & {128 bits} \\
			\hline
			LSH      & 0.4388 & 0.4393  & 0.4514  & 0.4553  & 0.4773   \\
			S-RBM    & 0.4846 & 0.5108  & 0.5166  & 0.5190  & 0.5137   \\
			SpH      & 0.5807 & 0.6055  & 0.6281  & 0.6143  & 0.5891   \\
			STH      & 0.3723 & 0.3947  & 0.4105  & 0.4181  & 0.4123   \\
			VDSH     & 0.4330 & 0.6853  & 0.7108  & 0.4410  & 0.5847   \\
			NASH     & 0.5849 & 0.6573  & 0.6921  & 0.6548  & 0.5998   \\
			BMSH     & n.a.    & 0.7062  & 0.7481  & 0.7519  & 0.7450   \\
			\hline
			Ours     &  \textbf{0.6959} &\textbf{0.7243}&\textbf{0.7534}  & \textbf{0.7606}& \textbf{0.7632}      \\   
			\bottomrule
	\end{tabular}}
	\caption{Precision of the top 100 retrieved documents on \textit{TMC} dataset.}
	\label{tb:tmc}
\end{table}

\begin{table}[t]
	\centering
	\resizebox{\columnwidth}{!}{    
		\begin{tabular}{c||c|c|c|c|c}
			
			\toprule
			{Method}   & {8 bits} & {16 bits} & {32 bits} & {64 bits} & {128 bits} \\
			\hline
			LSH        & 0.2802 & 0.3215  & 0.3862  & 0.4667  & 0.5194   \\  
			S-RBM      & 0.5113 & 0.5740  & 0.6154  & 0.6177  & 0.6452   \\  
			SpH        & 0.6080 & 0.6340  & 0.6513  & 0.6290  & 0.6045   \\  
			STH        & 0.6616 & 0.7351  & 0.7554  & 0.7350  & 0.6986   \\  
			VDSH       & 0.6859 & 0.7165  & 0.7753  & 0.7456  & 0.7318   \\  
			NASH       & 0.7113 & 0.7624  & 0.7993  & 0.7812  & 0.7559   \\  
			BMSH       & n.a.      & 0.7954  & 0.8286  & 0.8226  & 0.7941   \\  
			\hline
			Ours     &  \textbf{0.7589} &\textbf{0.8212}&\textbf{0.8420}  & \textbf{0.8465}& \textbf{0.8482}         \\
			\bottomrule
			
	\end{tabular}}
	\caption{Precision of the top 100 retrieved documents on \textit{Reuters} dataset.}
	\label{tb:Reuters}
	\vspace{-4mm}
\end{table}

\begin{table}[t]
	\centering
	\resizebox{\columnwidth}{!}{    
		\begin{tabular}{c||c|c|c|c|c}
			\toprule
			{Method}   & {8 bits} & {16 bits} & {32 bits} & {64 bits} & {128 bits} \\
			\hline
			LSH            & 0.0578 & 0.0597  & 0.0666  & 0.0770  & 0.0949   \\   
			S-RBM          & 0.0594 & 0.0604  & 0.0533  & 0.0623  & 0.0642   \\   
			SpH            & 0.2545 & 0.3200  & 0.3709  & 0.3196  & 0.2716   \\   
			STH            & 0.3664 & 0.5237  & 0.5860  & 0.5806  & 0.5443   \\   
			VDSH           & 0.3643 & 0.3904  & 0.4327  & 0.1731  & 0.0522   \\   
			NASH           & 0.3786 & 0.5108  & 0.5671  & 0.5071  & 0.4664   \\   
			BMSH           & n.a.      & 0.5812  & 0.6100  & 0.6008  & 0.5802   \\   
			\hline
			Ours     &  \textbf{0.4389} &\textbf{0.5839}&\textbf{0.6183}  & \textbf{0.6279}& \textbf{0.6359}   \\      
			\bottomrule
	\end{tabular}}
	\caption{Precision of the top 100 retrieved documents on \textit{20Newsgroups} dataset.}
	\label{tb:20Newsgroups}
\end{table}

\begin{figure*}[t]
	\begin{center}
		\begin{subfigure}[b]{0.325\textwidth}        
      \centering
      {\includegraphics[scale=0.600]{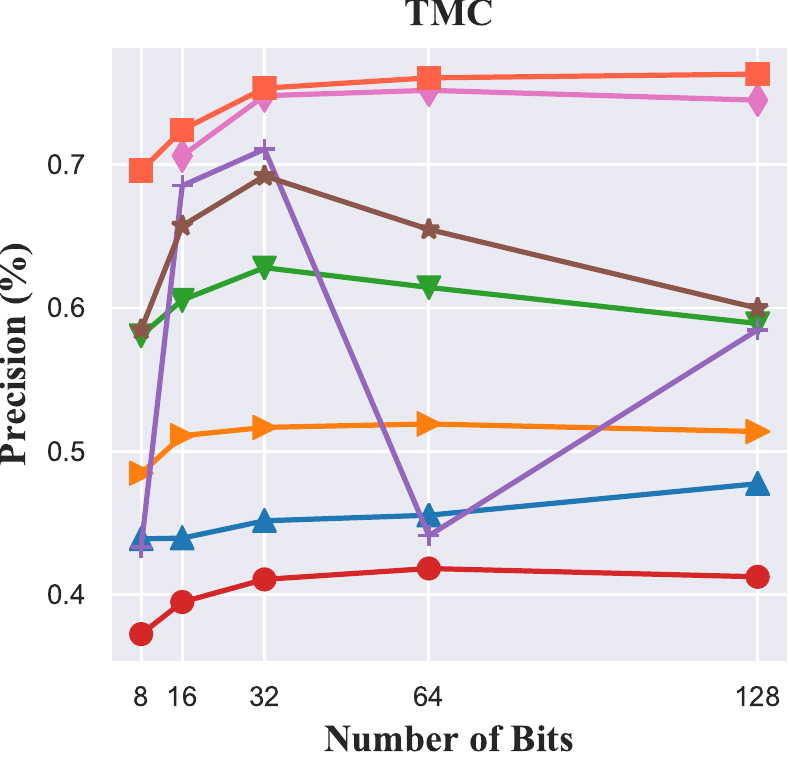}}
		\end{subfigure}
		\begin{subfigure}[b]{0.325\textwidth}        
      \centering
      {\includegraphics[scale=0.600]{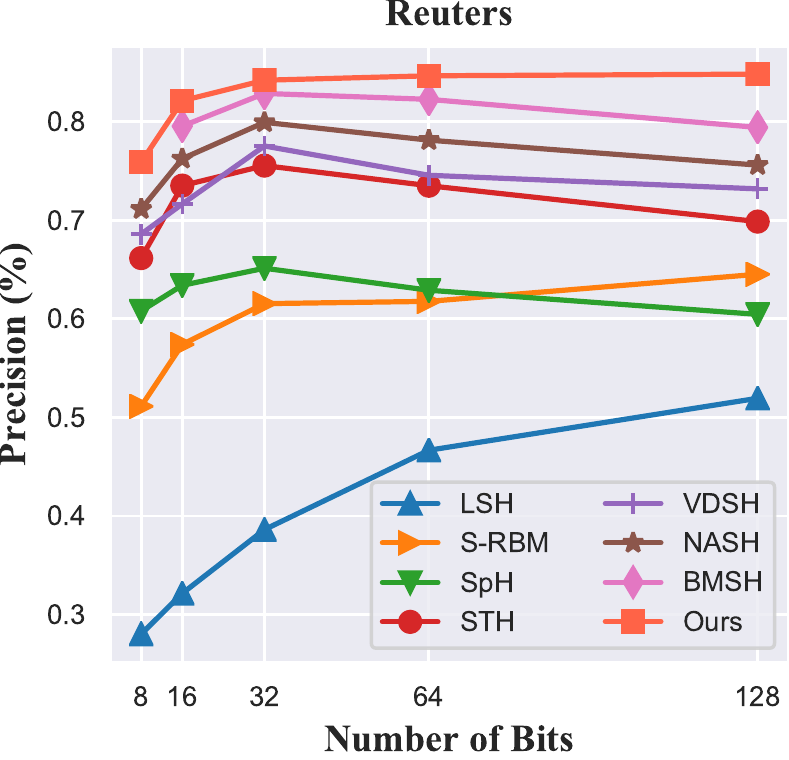}}
		\end{subfigure}
		\begin{subfigure}[b]{0.325\textwidth}        
      \centering
      {\includegraphics[scale=0.600]{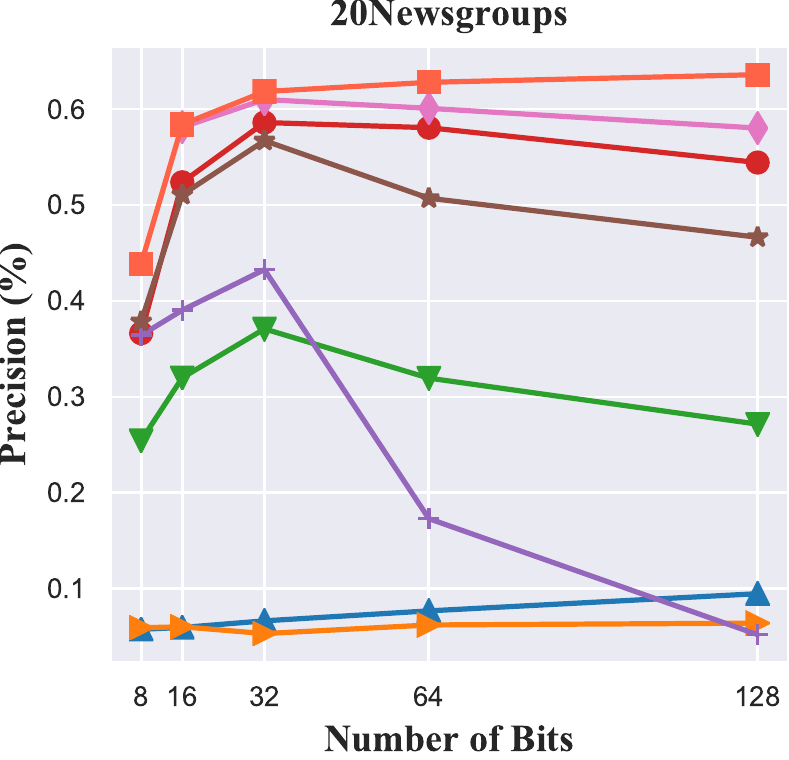}}
		\end{subfigure}
	\end{center}
	\caption{Retrieval precisions of unsupervised hashing methods 
		on three datasets under different code lengths.}
	\label{fig:1}
	
\end{figure*}

\subsection{Empirical Study of Computational Efficiency}
To show the computational efficiency of our proposed method, we also report the average running time per epoch in GPU on {\it TMC} dataset, which is of the largest among the considered ones, in Table \ref{tb:elapsed_time}. As a benchmark, the average training time of vanilla NASH is $2.553$s per epoch. It can be seen that because of to the use of low-rank parameterization of the covariance matrix, the proposed model can be trained almost as efficiently as vanilla NASH, but deliver a much better performance.
\begin{table}
	\centering
	\resizebox{0.8\columnwidth}{!}{    
		\begin{tabular}{c | c || c}
			\toprule
			Value of $v$ & Value of $k$  & Avg. Time (seconds) \\
			\hline
            1  &   1 & 2.934\\
            1  &   5 & 3.124\\
            5  &   1 & 3.137\\
            5  &  5 & 3.353\\
            10 &  5  & 3.403\\
            10 & 10 & 3.768\\
			\bottomrule
	\end{tabular}}
	\caption{Average running time per epoch on {\it TMC} dataset under different values of $v$ and $k$.}
	\label{tb:elapsed_time}
	\vspace{-4mm}
\end{table}

\subsection{Hash Codes Visualization}
\label{ssec:visualization}

\begin{figure*}
  \centering
  \begin{subfigure}[b]{0.32\textwidth} 
    \centering
    {\includegraphics[scale=0.45]{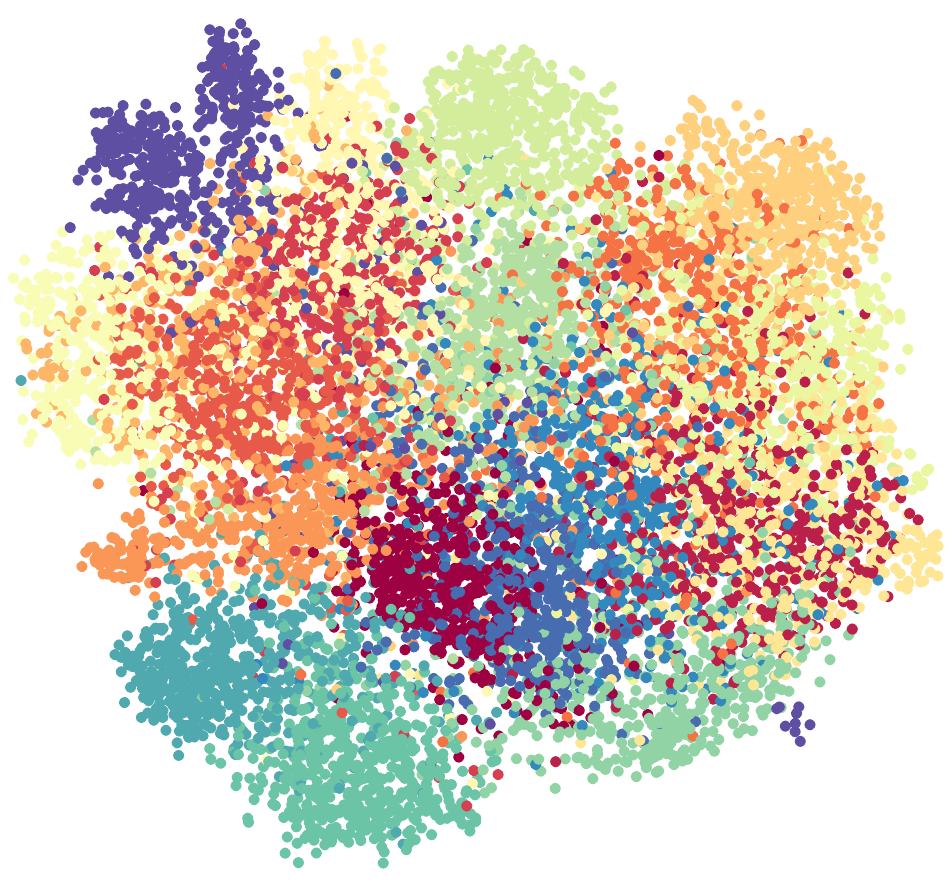}}
      \caption{VDSH}
  \end{subfigure}%
  \begin{subfigure}[b]{0.32\textwidth} 
    \centering
    {\includegraphics[scale=0.45]{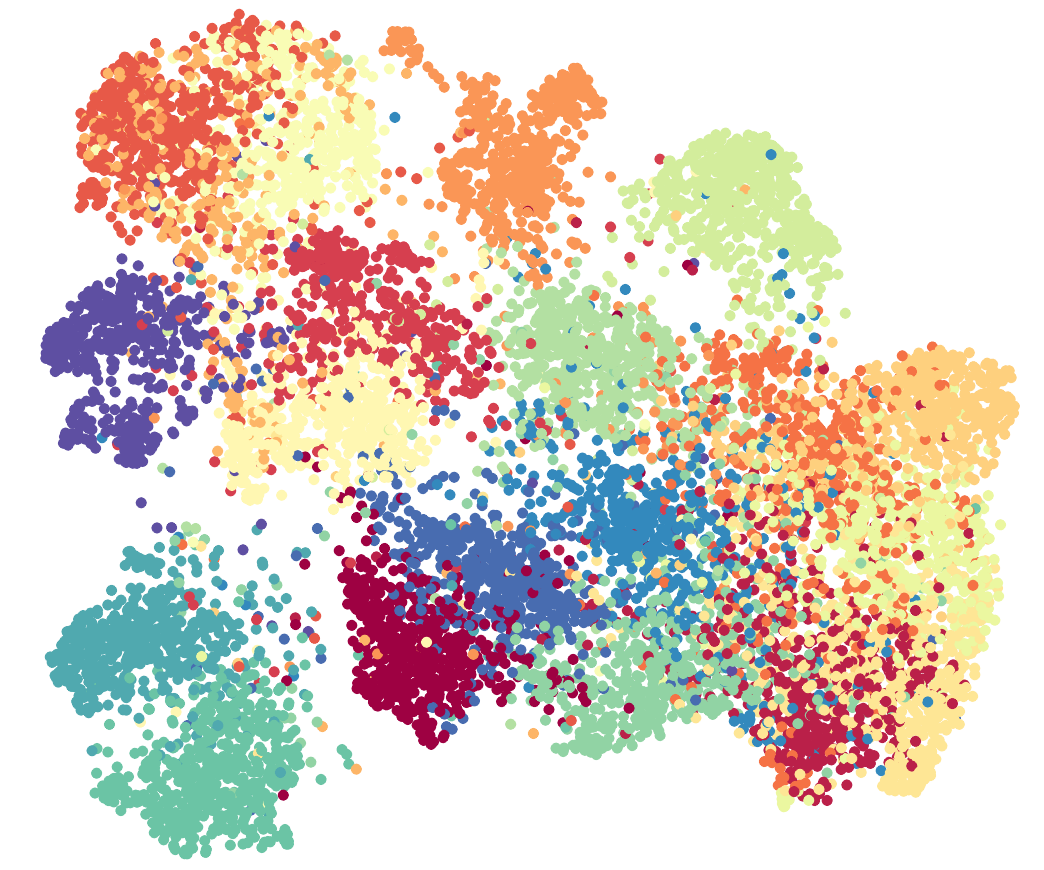}}
      \caption{NASH}
  \end{subfigure}%
  \begin{subfigure}[b]{0.32\textwidth} 
    \centering
    {\includegraphics[scale=0.45]{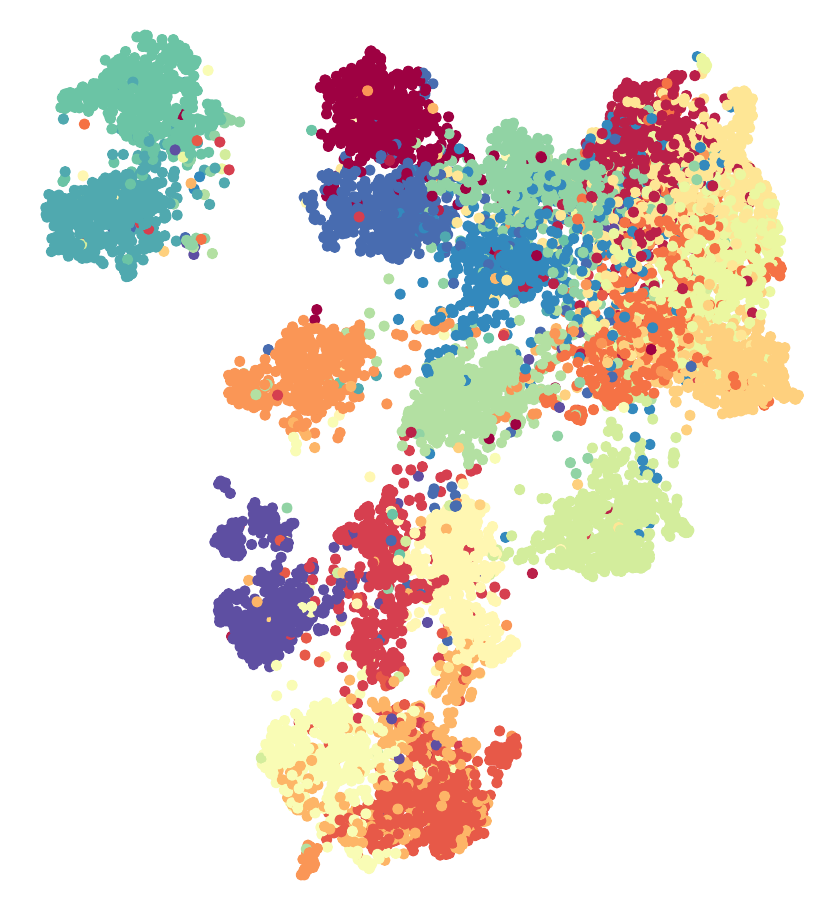}}
    
    \caption{Ours}
  \end{subfigure}
  \caption{Visualization of the 128-bit hash codes learned by VDSH, NASH and our model on 20Newsgroups dataset respectively. Each data point in the figure above denotes a hash code of the corresponding document, and each color represents one category.}
  \label{fig:2d_projections}
  \vspace{-3mm}
  \end{figure*}

To further investigate the capability of different models in generating semantic-preserving binary codes, we project the hash codes produced by VDSH, NASH and our proposed model on 20Newsgroups datasets onto a two-dimensional plane by using the widely adopted UMAP technique \citep{umap} and then visualize them on the two-dimensional planes, as shown in Figure \ref{fig:2d_projections}. It can be seen that the hash codes produced by VDSH are quite mixed for documents from different categories, while those produced by NASH are more distinguishable, consistent with the hypothesis that NASH is able to produce better codes than VDSH thanks to the end-to-end training. From the figure, we can further observe that the hash codes produced by our proposed method are the most distinguishable among all three methods considered, corroborating the benefits of introducing correlations among the bits of hash codes.

\subsection{Analyses on the Impacts of $v$ and $k$}
\label{ssec:analysis}

\paragraph{Ranks $v$}
Low-rank perturbed covariance matrix enables the proposed model to trade-off between  complexity and performance. That is, larger $v$ allows the model to capture more dependencies among latent variables, but the required computational complexity also increases. To investigate its impacts, we evaluate the performance of the 64-bit hash codes obtained from the proposed model under different values of $v$, with the other key parameter $k$ fixed to 10. The result is listed in the left half of Table \ref{tb:mk}. Notably, the proposed model with $v = 0$ is equivalent to NASH since there is not any correlation between the binary random variables. It can be seen that as the number of ranks increases, the retrieval precisions also increase, justifying the hypothesis that  employing the posteriors with correlations can increase the model's representational capacity and thereby improves the hash codes' quality in turn. It is worth noting that the most significant performance improvement is observed between the models with $v=0$ and $v=1$, and then as the value of $v$ continues to increase, the improvement becomes relatively small. This indicates that it is feasible to set the $v$ to a relatively small value to save computational resources while retaining competitive performance.

\paragraph{The number of mixture components $k$}
As stated in Section \ref{sec:new_bound}, increasing the number of components $k$ in the mixture distribution $h_k(s)$ will reduce the gap between the lower bound ${\mathcal{\widetilde L}}_k$ and the ELBO ${\mathcal{L}}$. To investigate the impacts of $k$, the retrieval precisions of the proposed model are evaluated under different values of $k$, while setting the other key parameter $v=10$. It can be seen from the right half of Table \ref{tb:mk} that as the number of components $k$ increases, the retrieval precision also increases gradually, suggesting that a tighter lower bound ${\mathcal{\widetilde L}}_k$ can always indicate better hash codes. Hence, if more mixture components are used, better hash codes can be expected. Due to the sake of complexity, only 10 components are used at most in the experiments.

\begin{table}
	\centering
	\resizebox{\columnwidth}{!}{    
		\begin{tabular}{c |c || c|c}
			\toprule
			Value of $v$   & {Precision } &  Value of $k$ & {Precision}\\
			\hline
			0 & 0.7812 & 1&0.8300\\
			1 & 0.8353 &   3&0.8391\\
			5 & 0.8406  &  5&0.8395\\
			10 & 0.8465 & 10&0.8465\\
			\bottomrule
	\end{tabular}}
	\caption{Left: Retrieval precisions under different values of $v$ with $k$ fixed to be 10 on \textit{Reuters} dataset; Right: Retrieval precision under different values of $k$ with $v$ fixed to be 10 on \textit{Reuters} dataset.}
	\label{tb:mk}
	\vspace{-4mm}
\end{table}

\section{Conclusion}
In this paper, by employing the distribution of Boltzmann machine as the posterior, we show that correlations can be efficiently introduced into the bits. To facilitate training, we first show that the BM distribution can be augmented as a hierarchical concatenation of a Gaussian-like distribution and a Bernoulli distribution. Then, an asymptotically-exact lower bound of ELBO is further developed to tackle the tricky normalization term in Boltzmann machines. Significant performance gains are observed in the experiments after introducing correlations into the bits of hash codes.
\label{sec:conclusion}

\section*{Acknowledgements}
This work is supported by the National Natural Science Foundation of China (NSFC) (No. 61806223, U1711262, U1501252, U1611264, U1711261), National Key R\&D Program of China (No. 2018YFB1004404), and Fundamental Research Funds for the Central Universities (No. 191gjc04). Also, CC appreciates the support from Yahoo! Research.

\bibliography{anthology,acl2020}
\bibliographystyle{acl_natbib}

\clearpage
\appendix

\section{Appendices}
\subsection{Proof of Proposition 1}
\label{appendix:proof_of_proposition_1}
\begin{proof}
	Making use of completing the square technique, the joint distribution of $r$ and $s$ can be decomposed as:
	\begin{align}
		&q(s,r) = q(s|r)q(r) \notag\\
	&= \frac{e^{-\frac{1}{2}(r-\mu)^{\top}\Sigma^{-1}(r-\mu) + r^{\top}s}}{|2\pi\Sigma|^{\frac{1}{2}}Z} \notag\\
	&= \frac{e^{-\frac{1}{2}\left[r-(\Sigma s + \mu)\right]^{\top}\Sigma^{-1}\left[r-(\Sigma s + \mu)\right]}}{|2\pi\Sigma|^{\frac{1}{2}} Z}
	e^{\mu^{\top}s + \frac{1}{2}s^{\top}\Sigma s}\notag\\
	&= q(r|s)q(s),\notag
	\end{align}
	where
	\begin{align}
		q(r|s) &= \mathcal{N}(r;\Sigma s + \mu, \Sigma), \notag\\
		q(s) &= \frac{1}{Z} e^{\mu^{\top}s+\frac{1}{2}s^{\top}\Sigma s}.\notag
	\end{align}
	From above, we show that the marginal distribution $q(s)$ is a Boltzmann machine distribution. 
	\end{proof}

\subsection{Proof of Proposition 2}
\label{appendix:proofofpos2}
We show the following facts about the proposed lower bound of ELBO ${\mathcal{\widetilde L}}_k$.

First, \emph{For any integer $k$, we have ${\mathcal{\widetilde L}}_{k+1} \geq {\mathcal{\widetilde L}}_k$.} For brevity we denote $\mathbb{E}_{q_\phi(r^{(1,\cdots, k)}|x )}$ as $\mathbb{E}_{r^{1..k}}$.
  First, due to the symmetry of indices, the following equality holds:
  \begin{equation}
    \begin{aligned}
      \mathbb{E}_{r^{1..k}} \mathbb{E}_{q(s|r^{(1)})} \!\log h_k(s)\!\! = \!\!\mathbb{E}_{r^{1..k}} \mathbb{E}_{q(s|r^{(i)})} \!\log h_k(s).\notag
    \end{aligned}
  \end{equation} 
  From this, we have
  \begin{equation}
    \begin{aligned}
      &\mathbb{E}_{r^{1..k}} \mathbb{E}_{q(s|r^{(1)})} \log h_k(s) \\
      &= \frac{1}{k}\sum_{i=1}^{k}\mathbb{E}_{r^{1..k}} \mathbb{E}_{q(s|r^{(1)})} \log h_k(s)\\
      &= \frac{1}{k}\sum_{i=1}^{k}\mathbb{E}_{r^{1..k}} \mathbb{E}_{q(s|r^{(i)})} \log h_k(s)\\
      &= \mathbb{E}_{r^{1..k}} \mathbb{E}_{h_k(s)} \log h_k(s),\notag
    \end{aligned}
  \end{equation}
  and
  \begin{equation}
    \begin{aligned}\label{exchangeable}
      &\mathbb{E}_{r^{1..k+1}} \mathbb{E}_{h_{k+1}(s)} \log h_{k+1}(s) \\
      &= \frac{1}{k+1}\sum_{i=1}^{k+1}\mathbb{E}_{r^{1..k+1}} \mathbb{E}_{q(s|r^{(i)})} \log h_{k+1}(s)\\
      &= \mathbb{E}_{r^{1..k+1}} \mathbb{E}_{q(s|r^{(1)})} \log h_{k+1}(s)\\
      &= \frac{1}{k}\sum_{i=1}^{k}\mathbb{E}_{r^{1..k+1}} \mathbb{E}_{q(s|r^{(i)})} \log h_{k+1}(s)\\
      &= \mathbb{E}_{r^{1..k+1}} \mathbb{E}_{h_k(s)} \log h_{k+1}(s).
    \end{aligned}
  \end{equation}
  Applying the equality \eqref{exchangeable} gives us:
   \begin{equation}
     \begin{aligned}
      &{\mathcal{\widetilde L}}_{k+1} - {\mathcal{\widetilde L}}_k \\
      &= \mathbb{E}_{r^{1..k}}\left[\mathrm{KL}(h_k(s) || q(s|x))\right] \\
      &\qquad -\mathbb{E}_{r^{1..k+1}}\left[\mathrm{KL}(h_{k+1}(s) || q(s|x))\right]\\
        &=\mathbb{E}_{r^{1..k+1}}\left[\mathrm{KL}(h_k(s) || q(s|x))\right.\\
      &\left. \qquad - \mathrm{KL}(h_{k+1}(s) || q(s|x))\right] \\
        &=\! \mathbb{E}_{r^{1..k+1}}\!\left[\mathbb{E}_{h_k(s)}\! \log\! h_k(s) \!\!-\!\!\mathbb{E}_{h_{k+1}(s)}\! \log\! h_{k+1}(s)\right] \\
      &= \mathbb{E}_{r^{1..k+1}}\left[\mathbb{E}_{h_k(s)} \!\log h_k(s)\! - \!\mathbb{E}_{h_k(s)} \!\log h_{k+1}(s)\right] \\
      &= \mathbb{E}_{r^{1..k+1}}\left[\mathrm{KL}(h_k(s) || h_{k+1}(s))\right] \geq 0.\notag
     \end{aligned}
   \end{equation}

We now show that $\lim_{k \rightarrow \infty}{\mathcal{\widetilde L}}_k = \mathcal{L}$. According to the strong law of large numbers, $h_k(s) = \frac{1}{k}\sum_{j}^k q(s|r^{(j)})$ converges to  $\mathbb{E}_{q(r|x)} \left[q(s|r)\right] = q(s|x)$ almost surely. We then have
    \begin{equation}
      \lim_{k \rightarrow \infty} \mathbb{E}_{r^{1..k}}\left[\mathrm{KL}(h_k(s) || q(s|x))\right] = 0.\notag
     \end{equation}
Therefore, $\mathcal{\widetilde L}_k$ approaches $\mathcal{L}$ as $k$ approaches infinity.

\subsection{Derivation of reparameterization for $h_k(s)$}
\label{appendix:mixturereparam}
Recall that $h_k(s) = \frac{1}{k}\sum_{j = 1}^k q(s|r_{\phi}^{(j)})$. We show that it can be easily reparameterized. Specifically, we could sample from such a mixture distribution through a two-stage procedure: {\it (i)} choosing a component $c \in\{1,2,\cdots,k\}$ from a uniform discrete distribution, which is then transformed as a $k$-dimensional {\it one-hot} vector $\tilde{c}$; {\it (ii)} drawing a sample from the selected component, {\it i.e.} $q(s|r_{\phi}^{(c)})$. Moreover, we define a matrix $R_{\phi}(x) \in \mathbb{R}^{m \times k}$ with its columns consisting of $r_{\phi}^{(1)}, r_{\phi}^{(2)}, \cdots, r_{\phi}^{(k)}$, each of which can be also reparameterized. In this way, a sample $\tilde  s_\phi$ from the distribution $h_k(s)$ can be simply expressed as
\begin{equation}
  \tilde  s_\phi = \frac{{\mathrm{sign}\left(\sigma(R_{\phi} \tilde{c}) -  u\right) +  1}}{2}\notag
\end{equation}
which can be seen as selecting a sample $r_{\phi}^{(c)}$ and then passing it through a perturbed sigmoid function. 
Therefore, during training, the gradients of $\phi$ are simply back-propagated through the chosen sample $r_{\phi}^{(c)}$.

\subsection{Comparisons between RBSH and our method}
\label{appendix:rbsh}
As discussed before, the main reason that we cited this paper but didn’t compare with it is that the datasets in \citep{rankingnash} are preprocessed differently as ours. Therefore, it is inappropriate to include the performance of the model from \citep{rankingnash} into the comparisons of our paper directly. Our work is a direct extension along the research line of VDSH and NASH. In our experiments, we followed their setups and used the preprocessed datasets that are publicized by them. However, in \citep{rankingnash}, the datasets are preprocessed by themselves. The preprocessing procedure influences the final performance greatly, as observed in the reported results. 

To see how our model performs compared to \citep{rankingnash}, we evaluate our model on the 20Newsgroup and TMC datasets that are preprocessed by the method in \citep{rankingnash}. The results are reported in Table \ref{tb:supp_rbsh}, where RBSH is the model from \citep{rankingnash}. We can see that using the same preprocessed datasets, our model overall performs better than RBSH, especially in the case of long codes. It should be emphasized that the correlation-introducing method proposed in this paper can be used with all existing VAE-based hashing models. In this paper, the base model is NASH, and when they are used together, we see a significant performance improvement. Since the RBSH is also a VAE-based hashing model, the proposed method can also be used with it to introduce correlations into the code bits, and significant improvements can also be expected.

\begin{table}
	\centering
	\resizebox{0.9\columnwidth}{!}{    
		\begin{tabular}{c||c|c||c|c}
        \toprule
         \multirow{2}{*}{\shortstack{Number\\ of Bits}} & \multicolumn{2}{c||}{20Newsgroup} & \multicolumn{2}{c}{TMC}  \\
         \cline{2-5}
         & RBSH  & Ours  & RBSH  & Ours \\
         \hline
         8   & 0.5190  & {\bf 0.5393}  & 0.7620  & {\bf 0.7667} \\
         16  & 0.6087  & {\bf 0.6275}  & 0.7959  & {\bf 0.7975} \\
         32  & 0.6385  & {\bf 0.6647}  & 0.8138  & {\bf 0.8203} \\
         64  & 0.6655  & {\bf 0.6941}  & 0.8224  & {\bf 0.8289} \\
         128 & 0.6668  & {\bf 0.7005}  & 0.8193  & {\bf 0.8324} \\
         \bottomrule
        \end{tabular}}
	\caption{Precision of the top 100 received documents on {\it 20Newsgroup} and {\it TMC} datasets.}
	\label{tb:supp_rbsh}
	\vspace{-4mm}
\end{table}
\end{document}